\documentclass[12pt]{article}
\title{Nonlinear Evolutionary PDE-Based Refinement of Optical Flow}
\author{Hirak Doshi, N. Uday Kiran}
\date{}
\usepackage{mathtools}
\usepackage{amsmath,amsfonts,amssymb,amsthm}
\usepackage{relsize}
\usepackage{xcolor}
\usepackage{subfig}
\usepackage{float}
\usepackage{multirow}
\usepackage{algorithm}
\usepackage[noend]{algpseudocode}

\usepackage[a4paper,left=3cm,top=3cm,right=3cm,bottom=3cm]{geometry}
\usepackage[utf8]{inputenc}

\usepackage{array}
\newcolumntype{L}{>{\centering\arraybackslash}m{11cm}}

\usepackage[T1]{fontenc}
\usepackage{tikz}
\usepackage{setspace}
\usetikzlibrary{shapes.arrows}
\usetikzlibrary{arrows,decorations.pathreplacing,positioning}
\theoremstyle{definition} % Add this line to make normal font instead of italics globally
\newtheorem{remark}{Remark}
\newtheorem{lemma}{Lemma}

\newtheorem*{keywords}{Keywords}

\newtheorem*{class}{Mathematics Subject Classification 2020}

\DeclareMathOperator{\prox}{\textbf{prox}}
\DeclareMathOperator*{\argmin}{arg\,min}
\DeclareMathOperator*{\proj}{\textbf{proj}}
\DeclareMathOperator*{\argmax}{arg\,max}
\usepackage{hyperref}
\hypersetup{
	colorlinks=true,
	linkcolor=blue,
	citecolor=red,      
	urlcolor=cyan,
}
\allowdisplaybreaks
\begin{document}
	\onehalfspacing
	\fontsize{13pt}{13pt}\selectfont
	\maketitle
	\begin{abstract}
The goal of this paper is to propose two nonlinear variational models for obtaining a refined motion estimation from an image sequence. Both the proposed models can be considered as a part of a generalized framework for an accurate estimation of physics-based flow fields such as rotational and fluid flow. The first model is novel in the sense that it is divided into two phases: the first phase obtains a crude estimate of the optical flow and then the second phase refines this estimate using additional constraints. The correctness of this model is proved using an evolutionary PDE approach. The second model achieves the same refinement as the first model, but in a standard manner, using a single functional. A special feature of our models is that they permit us to provide efficient numerical implementations through the first-order primal-dual Chambolle-Pock scheme. Both the models are compared in the context of accurate estimation of angle by performing an anisotropic regularization of the divergence and curl of the flow respectively. We observe that, although both the models obtain the same level of accuracy, the two-phase model is more efficient. In fact, we empirically demonstrate that the single-phase and the two-phase models have convergence rates of order $O(1/N^2)$ and $O(1/N)$ respectively. 
	\end{abstract}
	\begin{keywords}
Optical Flow, Evolutionary PDE, Variational Methods, Primal-Dual, Convergence
	\end{keywords}
\begin{class}
	35A15, 35J47, 35Q68.
\end{class}
\section{Introduction}
\label{sec:1}

Optical flow plays a key role in many advanced Computer Vision applications. It is a rich source of information on perceptible motion in our visual world. It's reliable estimation is thus important and at the same time challenging. Assuming the principle of local conservation of intensity and small temporal variations, optical flow involves the recovery of a function $\textbf{u}=(u,v)$ such that
\[
f(\textbf{x},\tau) = f(\textbf{x}+\textbf{u},\tau+\Delta \tau).
\]
where $f:\Omega\times [0,T]\to \mathbb{R}$ is the image sequence,  $\Omega\subset \mathbb{R}^2$ is open and bounded. This establishes a correspondence between pixel motions. Using first-order approximations the above relation can be written as

\begin{equation}
	f_\tau+\nabla f\cdot \textbf{u}=0,
	\label{ofc}
\end{equation}
which is widely known as the Optical Flow Constraint (OFC). To recover the velocity components using a variational minimization approach one writes (\ref{ofc}) as
\begin{equation*}
	\min_{\textbf{u}} J_1(\textbf{u}) = \int_\Omega (f_t+\nabla f\cdot \textbf{u})^2.
\end{equation*} 
This is the simplest least-square minimization. This problem is ill-posed as it leads to the aperture problem. Additional regularization terms are necessary to ensure well-posedness. The most common regularization term is the quadratic smoothness penalizing the gradient of the components of the flow originally introduced by Horn and Schunck \cite{Horn} in their seminal work. Cohen \cite{Cohen} and Kumar et.al. \cite{kumar} used the $L^1$ regularization which is more robust to outliers and preserves important edge information. A new discontinuity preserving optical flow model with $L^1$ norm on the OFC was proposed and studied by Aubert et. al. \cite{Aubert} in the space of functions of bounded variations $BV(\Omega)\times BV(\Omega)$. The well-posedness of the Horn and Schunck model, as well as the Nagel model was studied by Schn\"{o}rr \cite{Schnorr} in the space $H^1(\Omega)\times H^1(\Omega)$. Taking a step further, the authors in \cite{burger} proposed a $L^p-TV/L^p$ ($p=1$ or $2$) model combining both $L^1$ and $L^2$ terms. The behaviour of their regularization term is similar to the Huber function:
\begin{equation*}
	H(x;\epsilon)=
	\begin{cases}
		\frac{x^2}{2\epsilon}, \quad 0\le |x| \le \epsilon \\[0.2cm]
		|x|-\frac{\epsilon}{2}, \quad |x|> \epsilon. 
	\end{cases}
\end{equation*}
A detailed review and rigorous analysis of several variational optical flow models within the framework of calculus of variations can be found in \cite{Hinterberger}. 

Though most of the estimation involving rigid or quasi-rigid motion can be handled by minimizing OFC with a suitable regularization, it is insufficient to provide an accurate estimation for fluid-based images. Traditional computer vision techniques may not be suitable to capture these deformations of brightness patterns because of the high spatio-temporal turbulence in these sequences. These reasons have motivated researchers to look for an alternative constraint that can not only preserve pixel-correspondence but also capture certain intrinsic features of the flow. This paradigm shift hints at constraints that are physics-dependent. A lot of work has been done involving physics-based constraints for fluid motion estimation \cite{Corpetti2,Corpetti3,Liu2,Liu3,Luttman}. In \cite{paper1}, we have proposed a constraint-based refinement of optical flow. Using an image-driven evolutionary PDE model resulting from a quadratic regularization we have shown the well-posedness of such a refinement principle. An important characteristic of the model is the possibility of a diagonalization by the Cauchy-Riemann operator leading to a decoupled system involving diffusion of the curl and a multiplicative perturbation of the laplacian of the divergence of the flow. For a specific case, it was shown that the model is close to the physics-based model \cite{Liu2} using a modified augmented Lagrangian method.

%A topic of major interest to the computer vision community is the performance evaluation methods for optical flow estimation. Several important contributions \cite{barron, brox,bruhn,zhang} discuss robust methods for obtaining high accuracy flow fields. Motivated by the harmonic-constraint based regularization \cite{zhao}, we proposed a variational model for capturing rotational features and improving angular accuracy of the flow \cite{Doshi}.

The current work proposes a unified framework for a nonlinear evolutionary PDE-based refinement of optical flow. The first model is a two-phase refinement process. A crude pixel correspondence is obtained in the first phase which constitutes a good starting solution. In the next phase, this estimate is refined using additional constraints. This constraint is chosen motivated by the non-conservation term $f\nabla\cdot\textbf{u}$ in the physics-based constraint, (see \cite{paper1}). The second model estimates the flow directly in a single phase. In this case the additional constraint is chosen motivated by the harmonic constraint-based regularization (see \cite{zhao}). This approach replaces the oriented smoothness constraint with a weighted decomposition of divergence and curl of flow. We aim to capture the rotational features better by preserving edge information and improving the accuracy of the flow. Thus we consider only the curl component in our framework with an anisotropic weight term.

The total variation regularization leads to $\Delta_1$, the 1-Laplacian operator in the Euler-Lagrange equations. Obtaining a stable convergent scheme is a difficult task because of the singularity of the operator at the origin. As a result most of the implementation methods often yield slower algorithms. In this direction important contributions were made by Chambolle \cite{cham1} and Zach \cite{zach}. Chambolle and Pock \cite{cham} proposed a first-order primal-dual algorithm for solving non-smooth convex optimization problems. This further opened up newer directions as a large class of problems in Image Processing and Computer Vision could be solved within this framework. In our work, we use the Chambolle-Pock algorithm for both of our models. The numerical implementation of the algorithm has two main steps, namely updating the primal variables by solving a system of equations at each iteration and updating the dual variables by computing the point-wise projection maps onto the unit ball. Both of these steps are computationally expensive. As a result the second model yields a slower algorithm with a convergence rate of order $O(1/N^2)$. By $O(1/N^2)$ we mean if $\epsilon$ is the error threshold then the number of iterations required to reach this threshold is $1/\epsilon^2$. The first model splits the above-mentioned steps in two-phases. This leads to a faster algorithm with a convergence rate of $O(1/N)$.

The organization of the paper is as follows. In Section (\ref{sec:2}) we give the general formulation and describe our model in detail. Next in Section (\ref{sec:3}), we study the mathematical well-posedness of our formulation using an evolutionary PDE approach. Subsequently, we employ the first-order primal-dual Chambolle-Pock algorithm to our models and derive the necessary optimality conditions in Section (\ref{sec:4}). We then discuss the implementation details, discretization of our models and empirically demonstrate the nature of convergence in Section (\ref{sec:5}).
 
\section{Our Model Description}
\label{sec:2}
Our general formulation is given as:
\begin{equation}
	J(\textbf{u}) = \int_\Omega \rho(|f_t+\nabla f\cdot \textbf{u}|)+\alpha\sum_{i=1}^2\int_\Omega\gamma(|\nabla u_i|)+\beta\int_\Omega \phi(x,f,\nabla f)\psi(\textbf{u},\nabla \textbf{u})
	\label{fun1}
\end{equation}
where $\rho:\mathbb{R}\to\mathbb{R}$ is a function of the optical flow constraint, $\gamma:\mathbb{R}\to\mathbb{R}$ governs the regularization of the flow. The functions $\phi$ and $\psi$ are chosen specific to applications. A summary of some of the variational models which belong to this framework is listed in the following table:
\begin{table}[H]
	\centering
	\begin{tabular}{|c|c|c|c|c|}
		\hline
		& $\rho(x)$ & $\gamma(x)$ & $\phi(x,f,\nabla f)$ &$\psi(\textbf{u},\nabla \textbf{u})$ \\
		\hline
		Horn and Schunck\cite{Horn} & $x^2$ & $x^2$ & 0 & 0 \\[0.2cm]
		Cohen\cite{Cohen} & $x^2$ & $x$ & 0 & 0 \\[0.2cm]
		Aubert\cite{Aubert} & $x$ & $\sqrt{1+x^2}$ & $c(x)$ & $\textbf{u}^2$\\[0.2cm]
		$L^1-TV$\cite{zach} & $x$ & $x$ & 0 & 0\\[0.2cm]
		Our Model (M1) & 0 & $x$ & $f^2$ & $(\nabla\cdot \textbf{u})^2$\\[0.2cm]
		Our Model (M2) & $x^2$ & $x$ & $\frac{\lambda^2}{\|\nabla f\|^2 + \lambda^2}$ & $(\nabla_H\cdot \textbf{u})^2$\\
		\hline
	\end{tabular}
\caption{Some choices for the functions}
\end{table}
Based on the choice functions mentioned in the above table, the constraint-based refinement formulation becomes
\begin{equation}
	(M1)\quad J(\textbf{u}) = \alpha \sum_{i=1}^2\int_\Omega |\nabla u_i| +\beta \int_\Omega f^2(\nabla\cdot \textbf{u})^2.
	\label{fun2}
\end{equation}
where $|\cdot|$ is the Euclidean norm. Starting with $\textbf{u}=\textbf{u}_0$, where $\textbf{u}_0$ is the Horn and Schunck optical flow the above formulation obtains a refinement of $\textbf{u}_0$ driven by the additional constraint $f\nabla\cdot \textbf{u}$. This term is the non-conservation term in the physics-based constraint due to non null-out of plane components \cite{heitz}. This constraint preserves the spatial characteristics and vorticities of the flow. Thus this model can extract flow information from fluid-based digital imagery much better.

From the choice functions, our second formulation is given as:
 \begin{equation}
 	(M2)\quad J(\textbf{u}) = \int_\Omega (f_t+\nabla f\cdot\textbf{u})^2+\alpha \sum_{i=1}^2\int_\Omega |\nabla u_i| +\beta \int_\Omega \frac{\lambda^2}{|\nabla f|^2 + \lambda^2}(\nabla_H\cdot \textbf{u})^2.
 	\label{fun2}
 \end{equation}
The additional constraint in this formulation is motivated by the harmonic-constraint based regularization discussed in \cite{zhao}. By associating an anisotropic weight term with the curl in our formulation there are two main advantages. First, we are able to get a precise estimation of the infinitesimal rotation within the regions. Secondly, we achieve a better alignment of small vectors in the flow. This leads to an overall improvement in the endpoint error.
	
In either case, the choice of $\phi$ decides the influence of the image term in the regularization process. If $\phi=1$, then the additional constraint term in the functional is flow-driven, i.e. independent of the influence of the image data. The weight parameters $\alpha,\beta$ play an important role in the regularization process. For rigid-body like motion which requires important edge-information to be preserved a higher value of $\alpha$ is preferred. For fluid-based images where there is less edge-prominence, we choose a higher $\beta$ value.

\section{Well-Posedness}
\label{sec:3}
In this section we discuss the mathematical well-posedness of the proposed formulation. Let us denote $\textbf{u}=(u_1,u_2)$. The space $W^{1,p}(\Omega), p>1$, is the reflexive Banach space

\begin{equation*}
	W^{1,p}(\Omega) = \{u\in L^p(\Omega):D^\alpha u\in L^p(\Omega), |\alpha|\le 1\}
\end{equation*}
with the usual norm
\begin{equation*}
	\|u\|_{W^{1,p}}=\Bigg(\sum_{|\alpha|\le 1}\|D^\alpha u\|^p_{L^p}\Bigg)^{1/p}, 1\le p < \infty.
\end{equation*}
To study the mathematical well-posedness of the proposed formulation we consider the following approximation.
\begin{equation}
	J_{p,\text{R}}(\textbf{u})=\beta\int_\Omega \phi(x,f,\nabla f)(\nabla\cdot\textbf{u})^2+\frac{\alpha}{p}\int_\Omega\{|\nabla u_1|^p+|\nabla u_2|^p\}, \quad 1<p<2
	\label{fun3}
\end{equation}
This functional being strictly convex in $W^{1,p}(\Omega)$ admits a unique minimizer. For this discussion we consider $\phi(x,f,\nabla f)=f^2$. The first important step is to show that $J_{p,\textbf{R}}$ converges to $J_{1,\textbf{R}}$ as $p\to 1$. For this, we refer to the discussion in Section 3.4 in \cite{martin2}.
\begin{lemma}
	\begin{equation}
		\lim_{p\to 1}\frac{1}{p}\int_\Omega |\nabla u|^p=\int_\Omega |\nabla u|.
		\label{limit}
	\end{equation}
\end{lemma} 
\begin{remark}
	As $J_{p,\text{R}}(\textbf{u})\to J_{1,\text{R}}(\textbf{u})$, $p\to1$, the corresponding Euler-Lagrange equations $A_p=\Delta_p$ associated with the regularization term also converges to $A_1=\Delta_1$.
\end{remark}
\begin{remark}
	The case $p=2$ leads to a linear diffusion-driven refinement process. We have previously studied and discussed this case in \cite{paper1}.
\end{remark}
The associated parabolic system corresponding to the Euler-Lagrange equations of (\ref{fun3}) are given as
\begin{equation}
	\begin{cases}
		\dfrac{\partial u_1}{\partial t}=\Delta_p u_1 + a_0 \dfrac{\partial}{\partial x}[f^2((u_1)_x+(u_2)_y)] \text{ in }\Omega\times (0,\infty),\\[0.5cm]
		\dfrac{\partial u_2}{\partial t}=\Delta_p u_2 + a_0 \dfrac{\partial}{\partial y}[f^2((u_1)_x+(u_2)_y)] \text{ in }\Omega\times (0,\infty),\\[0.5cm]
		u_1(x,y,0) = u_1^0 \text{ in } \Omega,\\[0.3cm]
		u_2(x,y,0) = u_2^0 \text{ in } \Omega,\\[0.3cm]
		u_1 = 0 \text{ on }\partial\Omega\times (0,\infty),\\[0.3cm]
		u_2 = 0 \text{ on }\partial\Omega\times (0,\infty),
	\end{cases}
	\label{par}
\end{equation}
where $(u_1^0,u_2^0)$ is the starting feasible solution obtained by the Horn and Schunck optical flow, $a_0=2\beta/\alpha$. Rewriting the system in an abstract form leads us to
\begin{equation}
	\begin{cases}
		\dfrac{d \textbf{u}}{d t}+\mathcal{A}_p\textbf{u} =0,\quad t>0,\\[0.3cm]
		\textbf{u}(0)=\textbf{u}_0 \in H^1(\Omega)^2.
	\end{cases}
	\label{pl3} 
\end{equation}
Here the operator $\mathcal{A}_p=A_p+F$ where
\begin{equation*}
	A_p\textbf{u}=-
	\begin{bmatrix}
		\Delta_p u_1 \\[0.3cm]
		\Delta_p u_2
	\end{bmatrix},
	\qquad
	F\textbf{u}= -a_0
	\begin{bmatrix}
		\dfrac{\partial}{\partial x}[f^2((u_1)_x+(u_2)_y)]\\[0.3cm]
		\dfrac{\partial}{\partial y}[f^2((u_1)_x+(u_2)_y)]
	\end{bmatrix},
\end{equation*}
We will show that both the operators $A_p$ and $F$ are maximal monotone in $W^{1,p}(\Omega)\cap L^2(\Omega)$ and $L^2(\Omega)$ respectively.

\begin{lemma}
	The operators $A_p$ and $F$ is maximal monotone in $W^{1,p}(\Omega)\cap L^2(\Omega)$ and $L^2(\Omega)$ respectively.
\end{lemma}
\begin{proof}
	The maximal monotonicity of $A_p$ follows directly from the discussions in \cite{wei}. To show monotonicity we show that $\langle F\textbf{u},\textbf{u}\rangle \ge 0$. Indeed,
	\begin{align*}
		\langle F\textbf{u},\textbf{u}\rangle &= -a_0\int_\Omega \Big\{\frac{\partial}{\partial x}[f^2((u_1)_x+(u_2)_y)]u + \frac{\partial}{\partial y}[f^2((u_1)_x+(u_2)_y)] v\Big\},\\
		& = a_0\int_\Omega \{f^2((u_1)_x+(u_2)_y)u_x + f^2((u_1)_x+(u_2)_y)v_y\}, \\
		& =a_0\int_\Omega f^2((u_1)_x+(u_2)_y)^2 \ge 0,
	\end{align*}
	proving the monotonicity of $F$. To show the maximality we have to show that 
	\begin{equation}
		\text{Ran}(I+F)=L^2(\Omega)^2,
		\label{ran}
	\end{equation}
	i.e. there exists $\textbf{u}$ for all $\textbf{f}\in L^2(\Omega)^2$ such that $\textbf{u}+F\textbf{u}=\textbf{f}$ holds. Let $\textbf{f}=(f,g)\in L^2(\Omega)^2$ and consider the system
	\begin{equation*}
		u+\frac{\partial}{\partial x}[f^2((u_1)_x+(u_2)_y)] = f,
		\label{1}
	\end{equation*}
	\begin{equation*}
		v+\frac{\partial}{\partial y}[f^2((u_1)_x+(u_2)_y)] = g, \label{2}	
	\end{equation*}
	where $f,g\in L^2(\Omega)$. Applying the Cauchy-Riemann operator
	\begin{equation*}
		R=\begin{bmatrix}
			\partial_y & -\partial_x \\[0.3cm]
			\partial_x & \partial_y	
		\end{bmatrix}.
	\end{equation*}
	on both sides we obtain the decoupled system
	\begin{align}
		(u_1)_y-(u_2)_x = f_y-g_x,\label{eq1}\\
		(\Delta\circ k)((u_1)_x+(u_2)_y)= f_x+g_y \label{eq2}
	\end{align}
	where $k$ is the image-dependent multiplicative function $k:f\mapsto1+a_0f^2$. The first equation (\ref{eq1}) governs the curl of the flow $\textbf{u}=(u_1,u_2)$. The second equation (\ref{eq2}) indicates a non-homogeneous weighted diffusion process on the divergence with a weight $k$. Let us define $h_1=f_y-g_x$ and $h_2=f_x+g_y$. Solving the second equation gives us an expression for the divergence of the flow. Let us call this as $h_3$. We thus obtain the following system
	\begin{align*}
		(u_1)_y-(u_2)_x &= h_1, \\[0.3cm]
		(u_1)_x+(u_2)_y &= h_3/k,
	\end{align*}
	These are the inhomogeneous Cauchy-Riemann equations. In a compact form we rewrite them as
	\begin{equation*}
		R\textbf{u} = \tilde{\textbf{f}}
	\end{equation*}
	where $\tilde{\textbf{f}}=(h_1,h_3/k)$. The operator $R^{-1}$ is a continuous operator of order -1 in the space $W^{1,p}(\Omega)$. Hence there exists a unique $\textbf{u}$ such that $\textbf{u}+F\textbf{u}=\textbf{f}$ holds. This concludes the proof.
\end{proof}
Now that we have shown the maximal monotonicity let us define a function $\Phi_p:L^2(\Omega)^2\to (-\infty,+\infty]$ by
\begin{equation*}
	\Phi_p(\textbf{u})=
	\begin{cases}
		J_{p,\textbf{R}}(\textbf{u}), \qquad \textbf{u}\in [W^{1,p}(\Omega)\cap L^2(\Omega)]^2 \\[0.3cm]
		+\infty, \qquad\quad \textbf{u}\in [L^2(\Omega)\setminus W^{1,p}(\Omega)]^2
	\end{cases}.
\end{equation*}
Then clearly $\Phi_p$ is convex and lower semi-continuous. Also $\Phi_p$ is proper since $D(\Phi_p)=D(A_p)\cap D(F)\neq \emptyset$. Thus the associated subdifferential $\partial\Phi(\textbf{u})\equiv \mathcal{A}_p$ is maximal monotone. Thus there is a unique solution $\textbf{u}$ of the inclusion
\begin{equation*}
	0 \in \textbf{u}'(t)+\partial\Phi_p(\textbf{u})
\end{equation*}
satisfying the initial conditions.
	\section{The Primal-Dual Framework}
	\label{sec:4}
The primal-dual method is a numerical tool for solving optimization problems. The main idea is to replace a primal problem with an equivalent saddle point problem by introducing dual variables and employ efficient algorithms to obtain the desired convergence. In the recent past several saddle point frameworks have been proposed for variational problems in image processing and computer vision \cite{cham,osher,zhang}. As our formulation involves non-smooth convex functionals, the most suitable framework is the one proposed by Chambolle and Pock \cite{cham}. Let $\Omega\subset \mathbb{R}^2$ be an open, bounded set, $\mathcal{X}, \mathcal{Y}$ be two finite-dimensional vector spaces with the scalar product $(\cdot,\cdot)$ and the norm $\|\cdot\|$. Denote the primal variable $\textbf{u}=(u_1,u_2)$ and the dual variable $\textbf{d}=(d_1,d_2,d_3)$. We first consider the variational problem in the following form
\begin{equation}
	\argmin_{\textbf{u}} G(\textbf{u}) + F(K\textbf{u}).
	\label{prim}
\end{equation}
where $F,G:\mathcal{X}\to [0,\infty]$ are convex, proper and lower-semicontinuous functionals, $K:\mathcal{X}\to \mathcal{Y}$ is a continuous, linear operator. The equivalent primal-dual formulation is given as
\begin{equation}
	\argmin_{\textbf{u}}\argmax_{\textbf{d}}\: G(\textbf{u}) + (K\textbf{u},\textbf{d}) - F^*(\textbf{d}).
	\label{pd}
\end{equation}
where $F^*$ is the convex conjugate of $F$. Table (\ref{table}) gives a summary of each term of our model using the above notations. Given a $\tau,\sigma >0$, an initial $(\textbf{u}^0,\textbf{d}^0)\in \mathcal{X}\times \mathcal{Y}$, the Chambolle-Pock algorithm solves the saddle point problem (\ref{pd}) by the following algorithm:
\begin{align*}
	\textbf{d}^{k+1}&=\prox_{\sigma F^*}(\textbf{d}^k +\sigma K\bar{\textbf{u}}^k),\\
	\textbf{u}^{k+1} &= \prox_{\tau G}(\textbf{u}^k-\tau K^*\textbf{d}^{k+1}),\\
	\bar{\textbf{u}}_{k+1} &=2\textbf{u}_{k+1}-\textbf{u}_k \qquad \text{(over-relaxation)},
\end{align*}
where
\begin{equation*}
	\prox_{\tau G}(\textbf{d}) =(I+\tau\partial G)^{-1}(\textbf{d})= \argmin_{\textbf{u}}\Bigg\{\frac{1}{2}\|\textbf{u}-\textbf{d}\|^2+\tau G(\textbf{u})\Bigg\}
\end{equation*}
is the proximal or the resolvent operator. This can be thought of as a trade-off between minimizing $G$ and being close to $\textbf{d}$. We now employ the above algorithm for our problem and derive the necessary optimality conditions.
\subsection{Optimality Condition for Our Model $(M1)$}
In this case we have
\[
G(\textbf{u})= 0, \quad F(K\textbf{u})=\frac{1}{2}\int_\Omega f^2(\nabla\cdot\textbf{u})^2 +\sum_{i=1}^2 \int_\Omega |\nabla u_i|.
\]
The Operator $K$ is given as
\begin{equation*}
	K\textbf{u}=\begin{bmatrix}
		\nabla & 0 \\
		0 & \nabla \\
		f^2\partial_x & f^2\partial_y
	\end{bmatrix}
	\begin{bmatrix}
		u_1\\[0.4cm] u_2
	\end{bmatrix}.
\end{equation*}
Therefore,
\begin{equation*}
	K^*\textbf{d} = -\begin{bmatrix}
		\nabla\cdot & 0 & \partial_x(f^2\cdot) \\
		0 & \nabla\cdot & \partial_y(f^2\cdot)
	\end{bmatrix}
	\begin{bmatrix}
		d_1\\
		d_2\\
		d_3
	\end{bmatrix}.
\end{equation*}

\begin{table}[t]
	\hspace{-2.3cm}
	\setlength{\tabcolsep}{0.7em} % for the horizontal padding
	{\renewcommand{\arraystretch}{1.0}
		\begin{tabular}{|p{2.0cm}|c|c|c|c|}
			\hline
			Model & $\phi(x,f,\nabla f)$ & $G(\textbf{u})$ & $F(K\textbf{u})$ & $K$\\
			\hline
			Model (M1) & $f^2$ & 0 & 
			$\begin{aligned}[t]
				\frac{1}{2}\int_\Omega \phi(x,f,\nabla f)(\nabla\cdot\textbf{u})^2 \\
				&\hspace*{-1.2cm}+\sum_{i=1}^2 \int_\Omega |\nabla u_i|
			\end{aligned}$ &
			$\begin{pmatrix}
				\nabla & 0\\
				0 & \nabla \\
				\phi\partial_x & \phi\partial_y
			\end{pmatrix}
			$\\
			
			Model (M2) & $\dfrac{\lambda^2}{\lambda^2+\|\nabla f\|^2}$ &$\displaystyle\frac{1}{2}\int_\Omega (f_t+\nabla f\cdot \textbf{u})^2$ & 	$\begin{aligned}[t]
				\frac{1}{2}\int_\Omega \phi(x,f,\nabla f)(\nabla_H\cdot\textbf{u})^2 \\
				&\hspace*{-1.2cm}+\sum_{i=1}^2 \int_\Omega |\nabla u_i|
			\end{aligned}$ & 
			$
			\begin{pmatrix}
				\nabla & 0\\
				0 & \nabla \\
				\phi\partial_y & -\phi\partial_x
			\end{pmatrix}
			$\\
			\hline 
		\end{tabular}
	}
	\caption{Summary of the terms in the Primal-Dual Formulation}
	\label{table}
\end{table}
\hspace*{-0.5cm}Using standard dual identities the convex conjugate $F^*(\textbf{d})$ is computed as
\begin{equation*}
	F^*(\textbf{d}) = \frac{1}{2}\|d_3\|^2_2 +\alpha\sum_{i=1}^2 \delta_{B(L^\infty)}(d_i/\alpha),
\end{equation*}
where $B(L^\infty)$ denotes the unit ball in $L^\infty(\Omega)$ and
\begin{equation*}
	\delta_{B(L^\infty)}(x^*)=\begin{cases}
		0, \text{ if } x^*\in B(L^\infty)\\
		+\infty, \text{ otherwise}
	\end{cases}.
\end{equation*}
Thus the primal-dual formulation is given as
\begin{equation*}
	\argmin_\textbf{u} \argmax_{\textbf{d}}\:\: ( \textbf{u},K^*\textbf{d}) -\frac{1}{2\beta}\|d_3\|^2_2 -\alpha\sum_{i=1}^2 \delta_{B(L^\infty)}(d_i/\alpha).
\end{equation*}
Accordingly, the Chambolle-Pock algorithm for this primal-dual problem is given as:
\begin{align*}
	\tilde{\textbf{d}}^{k+1}&=\textbf{d}^k+\sigma K\bar{\textbf{u}},\\
	\textbf{d}_{1,2}^{k+1}&=\argmin_{\textbf{d}}\Bigg\{\frac{1}{2}\|\textbf{d}-\tilde{\textbf{d}}_{1,2}^{k+1}\|_2^2+\alpha\sigma\delta_{B(L^\infty)}(\textbf{d}/\alpha)\Bigg\},\\
	d_{3}^{k+1}&=\argmin_{d}\Bigg\{\frac{1}{2}\|d-\tilde{d}_{3}^{k+1}\|_2^2+\frac{\sigma}{2\beta}\|d\|_2^2\Bigg\},\\
	\tilde{\textbf{u}}^{k+1}&=\textbf{u}^k-\tau K^*\textbf{d}^{k+1},\\
	\bar{\textbf{u}}_{k+1} &=2\textbf{u}_{k+1}-\textbf{u}_k.
\end{align*}
To derive the optimality condition for the dual variables $d_3$, consider the functional
\begin{equation*}
	J(d_3)=\frac{1}{2}\int_\Omega (d_3-\tilde{d}_3)^2+\frac{\sigma}{2\beta}\int_\Omega d_3^2.
\end{equation*}
Therefore setting $d_\theta J = 0$ we get
\begin{equation*}
	d_3-\tilde{d}_3+\frac{\sigma}{\beta}d_3 =0.
\end{equation*}
Rearranging we get
\begin{equation*}
	d_3^{k+1}=\frac{\beta}{\beta+\sigma}\tilde{d}_{3}^{k+1},
\end{equation*}
The solution for the indicator function $\delta$ is given by the point-wise projections of $\tilde{\textbf{d}}^{k+1}, \proj\nolimits_\alpha(\tilde{\textbf{d}}^{k+1})$ onto the unit ball, see \cite{burger,dirks}. Thus the iterative scheme for the Chambolle-Pock is given as
\begin{align*}
	\tilde{\textbf{d}}^{k+1}&=\textbf{d}^k+\sigma K\bar{\textbf{u}},\\
	\textbf{d}_{1,2}^{k+1}&=\proj\nolimits_\alpha\Big(\tilde{\textbf{d}}_{1,2}^{k+1}\Big),\\
	d_3^{k+1}&=\frac{\beta}{\beta+\sigma}\tilde{d}_{3}^{k+1},\\
	\tilde{\textbf{u}}^{k+1}&=\textbf{u}^k-\tau K^*\textbf{d}^{k+1},\\
	\bar{\textbf{u}}_{k+1} &=2\textbf{u}_{k+1}-\textbf{u}_k.
\end{align*}

\subsection{Optimality Condition for Our Model $(M2)$}
In this case we have
\[
G(\textbf{u})= \frac{1}{2}\int_\Omega (f_t+\nabla f\cdot \textbf{u})^2, \quad F(K\textbf{u})=\frac{1}{2}\int_\Omega \phi(f,\nabla f)(\nabla_H\cdot\textbf{u})^2 +\sum_{i=1}^2 \int_\Omega |\nabla u_i|.
\]
The Operator $K$ is given as
\begin{equation*}
	K\textbf{u}=\begin{bmatrix}
		\nabla & 0 \\
		0 & \nabla \\
		\phi\partial_y & -\phi\partial_x
	\end{bmatrix}
	\begin{bmatrix}
		u_1\\[0.4cm] u_2
	\end{bmatrix}.
\end{equation*}
Therefore,
\begin{equation*}
	K^*\textbf{d} = -\begin{bmatrix}
		\nabla\cdot & 0 & \partial_y(\phi\cdot) \\
		0 & \nabla\cdot & -\partial_x(\phi\cdot)
	\end{bmatrix}
	\begin{bmatrix}
		d_1\\
		d_2\\
		d_3
	\end{bmatrix}.
\end{equation*}
As before, the convex conjugate $F^*(\textbf{d})$ is computed as
\begin{equation*}
	F^*(\textbf{d}) = \frac{1}{2}\|d_3\|^2_2 +\alpha\sum_{i=1}^2 \delta_{B(L^\infty)}(d_i/\alpha),
\end{equation*}
Thus the primal-dual formulation is given as
\begin{equation*}
	\argmin_\textbf{u} \argmax_{\textbf{d}}\:\: \frac{1}{2}\int_\Omega (f_t+\nabla f\cdot \textbf{u})^2 + ( \textbf{u},K^*\textbf{d}) -\frac{1}{2\beta}\|d_3\|^2_2 -\alpha\sum_{i=1}^2 \delta_{B(L^\infty)}(d_i/\alpha).
\end{equation*}
Accordingly, the Chambolle-Pock algorithm for this primal-dual problem is given as:
\begin{align*}
	\tilde{\textbf{d}}^{k+1}&=\textbf{d}^k+\sigma K\bar{\textbf{u}},\\
	\textbf{d}_{1,2}^{k+1}&=\argmin_{\textbf{d}}\Bigg\{\frac{1}{2}\|\textbf{d}-\tilde{\textbf{d}}_{1,2}^{k+1}\|_2^2+\alpha\sigma\delta_{B(L^\infty)}(\textbf{d}/\alpha)\Bigg\},\\
	d_{3}^{k+1}&=\argmin_{d}\Bigg\{\frac{1}{2}\|d-\tilde{d}_{3}^{k+1}\|_2^2+\frac{\sigma}{2\beta}\|d\|_2^2\Bigg\},\\
	\tilde{\textbf{u}}^{k+1}&=\textbf{u}^k-\tau K^*\textbf{d}^{k+1},\\
	\textbf{u}^{k+1}&=\argmin_{\textbf{u}}\Bigg\{\frac{1}{2}\|\textbf{u}-\tilde{\textbf{u}}^{k+1}\|_2^2+\frac{\tau}{2}\int_\Omega (f_t+\nabla f\cdot \textbf{u})^2\Bigg\},\\
	\bar{\textbf{u}}_{k+1} &=2\textbf{u}_{k+1}-\textbf{u}_k.
\end{align*}
The optimality conditions for the dual variables follow directly from above. For the primal variable $\textbf{u}$ the optimality condition can be obtained directly by a quadratic minimization, see \cite{dirks} for more details. The equations are given as
\begin{align*}
	(1+\tau f_x^2)u_1 + \tau f_xf_y u_2 &= \tilde{u}^{k+1}_1 - \tau f_xf_t,\\
	\tau f_xf_yu_1 + (1+\tau f_y^2) u_2 &= \tilde{u}^{k+1}_2 - \tau f_yf_t.
\end{align*} 
Thus the iterative scheme for the Chambolle-Pock is given as
\begin{align*}
	\tilde{\textbf{d}}^{k+1}&=\textbf{d}^k+\sigma K\bar{\textbf{u}},\\
	\textbf{d}_{1,2}^{k+1}&=\proj\nolimits_\alpha\Big(\tilde{\textbf{d}}_{1,2}^{k+1}\Big),\\
	d_3^{k+1}&=\frac{\beta}{\beta+\sigma}\tilde{d}_{3}^{k+1},\\
	\tilde{\textbf{u}}^{k+1}&=\textbf{u}^k-\tau K^*\textbf{d}^{k+1},\\
	\textbf{u}^{k+1} &= \Bigg(\frac{b_1c_3-c_2b_2}{c_1c_3-c_2^2},\frac{b_2c_1-c_2b_1}{c_1c_3-c_2^2}\Bigg),\\
	\bar{\textbf{u}}_{k+1} &=2\textbf{u}_{k+1}-\textbf{u}_k.
\end{align*}
where $c_1,c_2,c_3$ are the elements of the coefficient matrix given by $c_1 = 1+\tau f_x^2,c_2 = \tau f_xf_y, c_3 = 1+\tau f_y^2$, $b_1,b_2$ are the right hand side values given by $b_1 = \tilde{u}^{k+1}_1 - \tau f_xf_t,b_2=\tilde{u}^{k+1}_2 - \tau f_yf_t$. In the next section we will look at the numerical discretization and other implementation details.
	\section{Results}
	\label{sec:5}
Having obtained the Chambolle-Pock algorithm for solving the saddle-point problem, we now look at the implementation details. Algorithm \ref{pd1} shows the Chambolle-Pock algorithm for our nonlinear constraint-based refinement model.

\begin{algorithm}[H]
	\caption{}\label{pd1}
	\begin{algorithmic}[1]
		\State Initialize $\tau,\sigma\gets 1/\sqrt 8,1/\sqrt 8$
		\State Initialize $\textbf{u}^0\gets$ HS($f_1,f_2$), $\textbf{d}^0\gets$ 0
		\State Initialize operator $K$
		\Repeat
		\State $\textbf{u}_\text{old}\gets\textbf{u}$
		\State Build Matrix $K$
		\State $\tilde{\textbf{d}}\gets \textbf{d}+\sigma K\bar{\textbf{u}}$
		\State $d_{1,2}\gets\proj_{\sigma/\alpha}(\tilde{d}_{1,2})$
		\State $d_3\gets \frac{\beta}{\beta+\sigma}\tilde{d}_3$
		\State Build Matrix $K^*$
		\State $\tilde{\textbf{u}}\gets \textbf{u}-\tau K^*\textbf{d}$
		\State $\textbf{u}\gets \tilde{\textbf{u}}$
		\State $\bar{\textbf{u}}\gets 2\textbf{u}-\textbf{u}_\text{old}$
		\Until{convergence}
	\end{algorithmic}
\end{algorithm}
As mentioned previously this model works in two phases wherein the first phase we obtain a crude-pixel correspondence and subsequently refine this estimate in the next phase driven by additional constraints. The initial Horn and Schunck flow was computed using the Chambolle-Pock algorithm, see \cite{cham,dirks}. Here we observed that using a forward difference scheme for both spatial and temporal image derivatives $f_x,f_y$ and $f_t$ respectively does not yield a stable discretization. Instead, a forward difference scheme for $f_t$ and a central difference scheme for $f_x,f_y$ does yield a stable numerical scheme. In the next step, the operator matrix $K$ is constructed for updating the dual variables $d_1,d_2,d_3$ shown in steps 9 and 10. This requires solving two sub-problems, one for $d_1,d_2$ and the other for $d_3$.

Now $\nabla u_i=(u_{{i}_x},u_{{i}_y}),i=1,2$. The associated dual variable is $d_i = (d_{i,1},d_{i,2})$. The primal formulation comprises of the total-variation regularization. Accordingly,
\begin{equation*}
	|\nabla u_i|_{L^1} = |u_{{i}_x}|+|u_{{i}_y}|.
\end{equation*}
Thus the associated dual norm for the variable $d_i$ gives
\begin{equation*}
	\|d_i\|_{L^\infty} = \max\{|d_{i,1}|,|d_{i,2}|\}.
\end{equation*}
The solution for this minimization is the point-wise projection onto the unit ball corresponding to the dual norm. As shown previously, the convex conjugate of the total variation term is the indicator function $\delta_{L^\infty}(d_i/\alpha)$. The associated convex set is defined by
\begin{equation*}
	\{d_i:\|d_i/\alpha\|\le 1\} = \{d_i:\|d_i\|\le \alpha\}, i = 1,2.
\end{equation*}
Thus the dual update for $d_{1,2}$ can be obtained by the point-wise projection of $\tilde{d}_{1,2}$ onto $[-\alpha,\alpha]$ given as
\begin{equation*}
	d_{1,2} = \proj\nolimits_{\sigma/\alpha}(\tilde{d}_{1,2}) = \min(\alpha,\max(-\alpha, \tilde{d}_{1,2})).
\end{equation*}
The sub-problem for $d_3$ is a linear quadratic minimization problem as discussed in the previous section. In the next step the adjoint operator $K^*$ is constructed to update the primal variable $\textbf{u}$. The subsequent over-relaxation step $\bar{\textbf{u}}\gets 2\textbf{u}-\textbf{u}_\text{old}$ is a particular case for $\theta=1$ in \cite{cham} for easier estimates of the convergence. The algorithm is further simplified if the regularization term is linear. In this case the only difference that occurs is the updation of the dual variable $d_{1,2}$ leading to the optimality condition
\begin{equation*}
	d_{1,2}=\frac{\alpha}{\alpha+\sigma}\tilde{d}_{1,2}.
\end{equation*}
The stopping criterion is determined by computing the normalized error from the primal and dual residues. This error metric was introduced by the authors in \cite{gold} and is numerically less expensive. Let $\textbf{u}^{(k)},\textbf{d}^{(k)}$ be the primal and the dual updates after $k$ iterations respectively. Then the primal and dual residues at the $k^{th}$ iteration are computed by the formula:
\begin{align*}
	p_{\text{res}}^{(k)} :&= \Bigg|\frac{\textbf{u}^{(k)}-\textbf{u}^{(k+1)}}{\tau} - K^*(\textbf{d}^{(k)}-\textbf{d}^{(k+1)})\Bigg|,\\
	d_{\text{res}}^{(k)} :&= \Bigg|\frac{\textbf{d}^{(k)}-\textbf{d}^{(k+1)}}{\sigma} - K(\textbf{u}^{(k)}-\textbf{u}^{(k+1)})\Bigg|.
\end{align*}
Therefore the normalized error at $k^{th}$ step is obtained as:
\begin{equation*}
	e^{(k)}=\frac{p_{\text{res}}^{(k)}+d_{\text{res}}^{(k)}}{\mu(\Omega)},
\end{equation*}
where $\mu(\Omega)$ refers to the measure of the domain $\Omega$. Chambolle and Pock \cite{cham} also showed that the convergence criterion is fulfilled when $\tau\sigma\|K\|^2< 1,\theta = 1$. Thus $\tau$ and $\sigma$ need to be chosen accordingly. An optimal numerical upper-bound was obtained by Chambolle \cite{cham1} which satisfies the above criterion. Accordingly we set $\tau=\sigma=1/\sqrt 8$.

The Chambolle-Pock algorithm for the angular accuracy model follows in a similar manner. The main difference lies in the primal update step 12 because of the explicit presence of the data term in the functional. The primal variable $\textbf{u}$ is updated by solving a quadratic minimization problem discussed previously leading to the following update step,
\begin{equation*}
	\textbf{u}^{k+1} = \Bigg(\frac{b_1c_3-c_2b_2}{c_1c_3-c_2^2},\frac{b_2c_1-c_2b_1}{c_1c_3-c_2^2}\Bigg),
\end{equation*}
where $c_1 = 1+\tau f_x^2,c_2 = \tau f_xf_y, c_3 = 1+\tau f_y^2$, $b_1 = \tilde{u}^{k+1}_1 - \tau f_xf_t,b_2=\tilde{u}^{k+1}_2 - \tau f_yf_t$. We now show the results for obtained by implementing the algorithm. The first sequence is the Oseen vortex pair. For more details on the sequence we refer to \cite{Liu2}.

\begin{figure}[H]%
	\centering
	\subfloat{{\includegraphics[width=6cm]{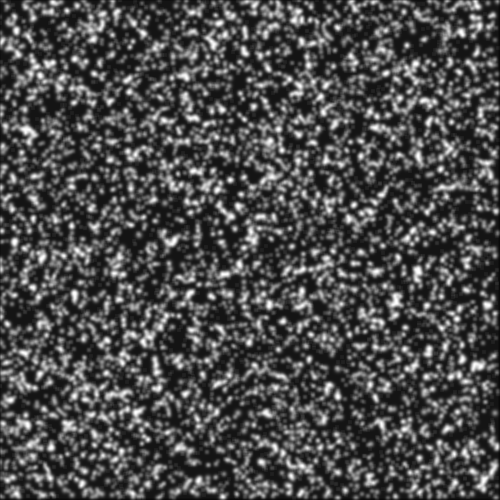}}}%
	\qquad
	\subfloat{{\includegraphics[width=6cm]{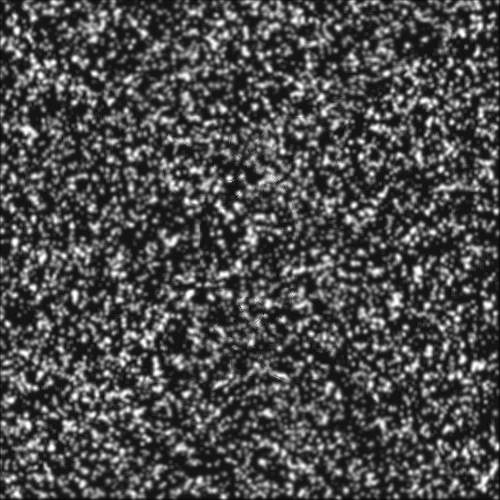}}}%
	\caption{Oseen Vortex Pair \cite{Liu2}.}%
	\label{f1}%
\end{figure}

\begin{figure}[H]%
	\centering
	\includegraphics[width=10cm]{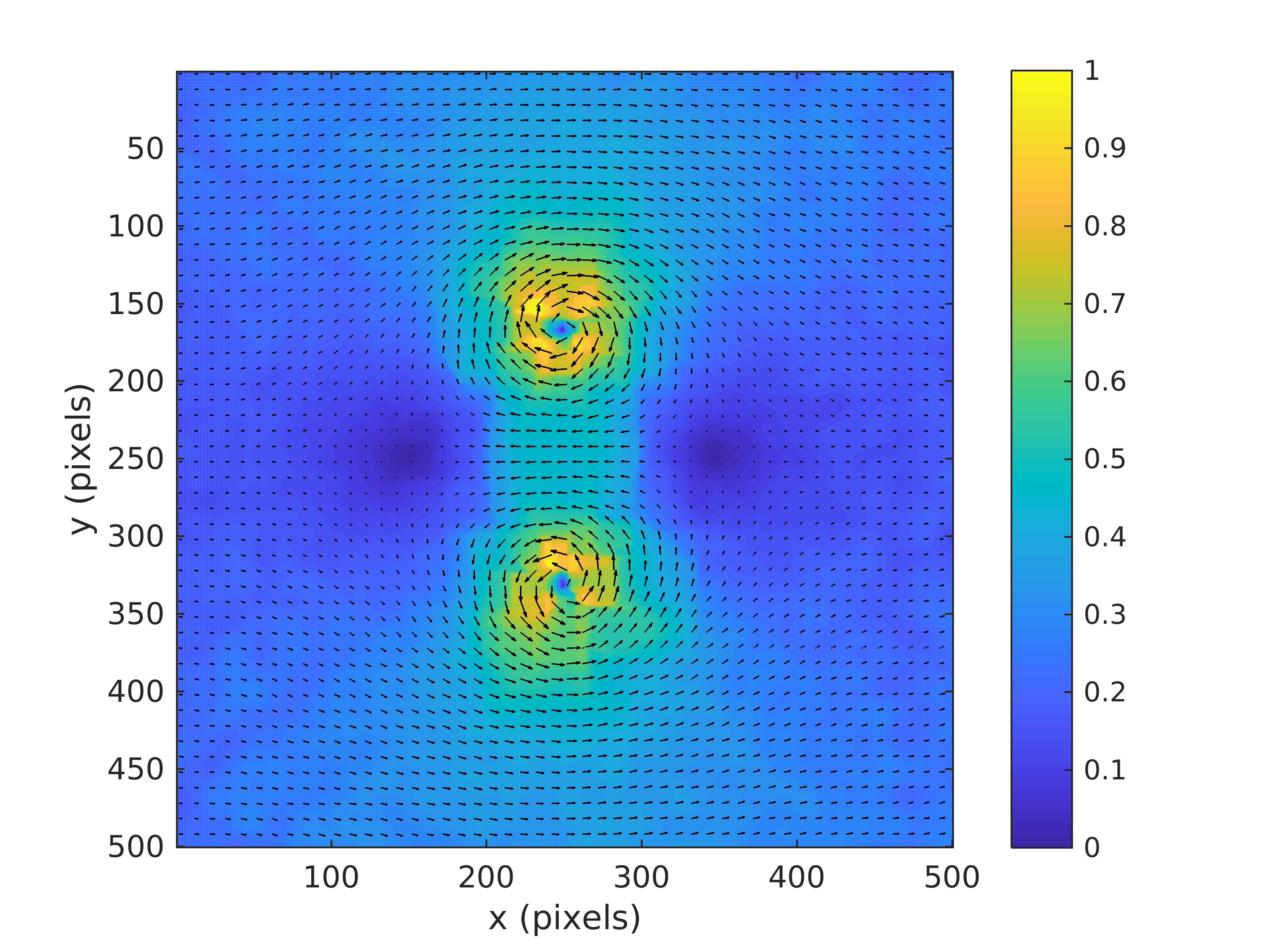}\label{1a}%
	\caption{Velocity magnitude plot for the Oseen vortex pair with $\alpha=0.1,\beta=0.01, \text{iter}=50$.}%
	\label{f2}%
\end{figure}
Figure (\ref{f2}) shows the velocity magnitude plot obtained for the oseen vortex pair. The algorithm produces dense flow fields while correctly estimating the vortex cores. We now empirically demonstrate the rate of convergence of both models. 

\begin{table}[H]
	\centering
	%\hspace*{-0.7cm}
	\renewcommand{\arraystretch}{1.2}
	\begin{tabular}{|p{4cm}|c|c|c|c|}
		\hline
		\multirow{2}{7cm}{} & \multicolumn{2}{c|}{$\epsilon=0.1$} & \multicolumn{2}{c|}{$\epsilon=0.01$}\\ 
		\cline{2-5}
		%\hline
		& Model (M1) & Model (M2) & Model (M1) & Model (M2) \\
		\cline{2-5}
		Oseen Vortex Pair & 78 & 627& 755 & 53753 \\ 
		Cloud Sequence  & 20 & 99& 450 & 16068 \\
		Sphere Sequence & 10 & 316& 561 & 18009\\
		Hydrangea & 103 & 346& 937 & 12574 \\
		Rubberwhale & 42 & 351& 617 & 13590 \\
		\hline
	\end{tabular}
\caption{Total number of iterations required by the algorithms to reach the threshold of $\epsilon$.} 
\label{table2}
\end{table}
Table (\ref{table2}) shows the number of iterations required by the algorithm to reach the error threshold of $\epsilon$. By an order $O(1/N)$ convergence we mean that the number of iterations required to reach a tolerance $\epsilon$ is $O(1/\epsilon)$. This is validated from the above table. For Model (M1), for $\epsilon=0.1$, the number of iterations required to reach the threshold of $0.1$ is a multiple of 10. For Model (M2) it requires a multiple of $10^2$ iterations. The table also shows that roughly it requires a multiple of $100$ iterations for the Model (M1) algorithm to reach the threshold of $\epsilon=0.01$ and a multiple of $100^2$ iterations for Model (M2). The reason for this efficiency can be explained from the fact that in phase 1, the Horn and Schunck initialization brings the solution within a close error range. As a result in the second phase we observe a $O(1/N)$ convergence as mentioned in \cite{cham}.

\subsection{Modern Implementation Principles}

Recent developments in optical flow computation reveal that the flow estimates can be significantly improved by incorporating certain established implementation principles. To accommodate these principles in our framework, Algorithm \ref{pd1} is suitably modified to make it a part of a larger implementation procedure.

The computation of image derivatives follows a weighted averaging principle \cite{sun1,bruhn}. The current flow estimates are used to warp the second image towards the first using bi-cubic interpolation. The time derivative is the difference between the first image and the warped image. The spatial derivatives are obtained as a weighted average of the first image and the warped image. The weight coefficient is called the blending ratio was chosen between 0 and 1.

\begin{figure}[H]%
	\centering
	\includegraphics[width=10cm]{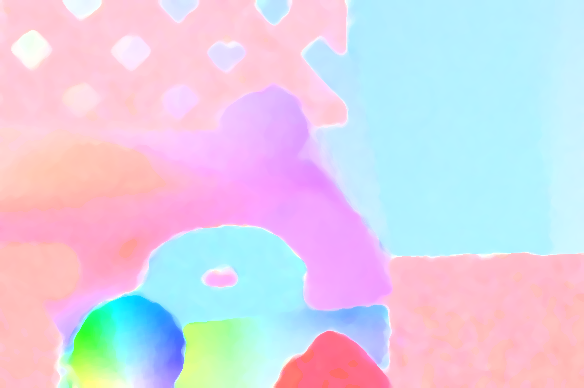}%
	\caption{Estimated flow field for the \emph{rubberwhale} sequence with $\alpha=10,\beta=1$.}%
	\label{f3}%
\end{figure}
To account for large displacements of pixel motions, a coarse-to-fine pyramidal scheme is employed \cite{wedel,bruhn,dirks}. The flow field is first computed at the coarsest level. This estimate is upsampled to the next finer level via interpolation and is used to warp the second image towards the first image. The flow increments at this finer level are then computed between the first image and the warped image. This process continues till the finest resolution level is reached. At each pyramid level, 10 warping steps are performed. after each warping iteration, a $5\times 5$ median filter is applied on the flow estimates to remove outliers. Figure (\ref{f3}) shows the obtained flow field for nonlinear refinement using Algorithm 1 for the rubberwhale sequence.

\begin{table}[H]
	\centering
	%\hspace*{-0.7cm}
	\renewcommand{\arraystretch}{1.2}
	\begin{tabular}{|p{4.8cm}|c|c|c|c|}
		\hline
		\multirow{2}{7cm}{} & \multicolumn{2}{c|}{Model 1$^\dag$} & \multicolumn{2}{c|}{Model 2$^\dag$} \\ 
		\cline{2-5}
		& \textbf{AAE} & \textbf{EPE} & \textbf{AAE} & \textbf{EPE} \\
		\hline
		
		Linear Refinement   & 3.412 & 0.105 & 3.410 & 0.105\\ 
		Nonlinear Refinement  & 3.397  & 0.104 & 3.355 & 0.103\\
		\hline
		
	\end{tabular}
	\caption{Comparison of the Average Angular Error (AAE) and End Point Error (EPE) for \emph{rubberwhale} sequence. ($^\dag$ refers to the algorithm incorporated with the modern implementation principles)}
	\label{t1}
\end{table}
An improvement is seen in the average angular error (AAE) and the end-point error (EPE) for the nonlinear refinement compared to the linear refinement for the \emph{rubberwhale} sequence as shown in Table (\ref{t1}). The improvement however is not significant which indicates that incorporating the modern implementation principles like coarse-to-fine warping, median filtering and so on also improves the accuracy of the linear refinement process. There are however image sequences for which edge-information are not well-preserved by the linear refinement process. We demonstrate this with the sphere sequence \cite{otago}.

\begin{figure}[H]%
	\centering
	\subfloat{{\includegraphics[width=5cm]{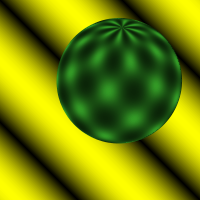}}}%
	\qquad
	\subfloat{{\includegraphics[width=5cm]{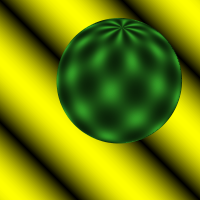}}}%
	\caption{Sphere sequence \cite{otago}}%
	\label{f6}%
\end{figure}

\begin{figure}[H]%
	\hspace{-1.5cm}
	\subfloat[Linear Refinement]{{\includegraphics[width=8cm]{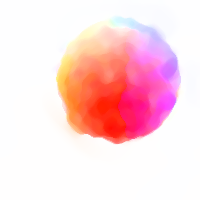}\label{5a}}}%
	\hspace{0.9cm}
	\subfloat[Nonlinear Refinement]{{\includegraphics[width=8cm]{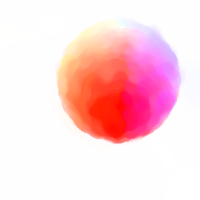}\label{5b}}}%
	\caption{Estimated flow fields of the \emph{sphere} sequence from \cite{otago} using Algorithm 1$^\dagger$ with $\alpha=1,\beta=0.1$ ($^\dagger$ refers to the algorithm incorporated with the modern implementation principles).}%
	\label{f5}%
\end{figure}

Figure (\ref{f5}) shows the color-coded flow estimate for the \emph{sphere} sequence using the Middlebury color coding. The isotropic behaviour is seen in the linear case because of which the edges are not well preserved.

\section*{Conclusion}
In this paper we have proposed two nonlinear variational models for obtaining an accurate estimation of physics-based flow fields such as rotational and fluid flow. The first model is a novel two-phase refinement process where in the first phase a crude estimate is obtained and subsequently refined using additional constraints in the second phase. We have studied the well-posedness of this model using an Evolutionary PDE approach. The second model performs the same refinement using a single functional. We used the first-order primal-dual Chambolle-Pock algorithm for the numerical implementation of the above models. We further empirically demonstrated that the two-phase model leads to a faster convergence rate of the order $O(1/N)$ compared to the second model which has a convergence rate of the order $O(1/N^2)$.
%\begin{acknowledgements}
%The authors dedicate this paper to Bhagawan Sri Sathya Sai Baba, Revered Founder Chancellor, SSSIHL.
%\end{acknowledgements}

\end{document}